\documentclass[letterpaper, 10 pt, conference]{formatting/ieeeconf}
\IEEEoverridecommandlockouts                              
\let\labelindent\relax
\usepackage{amsfonts}       

\usepackage{amsthm}
\usepackage{wrapfig}
\usepackage{mathtools}      
\usepackage{amssymb}        
\usepackage{filecontents}
\usepackage{graphicx}       
\usepackage{marginnote}     
\usepackage{marvosym}       
\usepackage{overpic}        
\usepackage{tabularx}
\usepackage{cite}
\usepackage{color}
\usepackage[linesnumbered,algoruled,boxed,lined]{algorithm2e}
\usepackage[normalem]{ulem}
\usepackage{epstopdf}
\usepackage{float}
\usepackage{enumitem}
\usepackage[normalem]{ulem}
\usepackage{lipsum}
\usepackage[a-2b,mathxmp]{pdfx}[2018/12/22]

\newif\ifdraft
\draftfalse

\ifdraft
\usepackage[paperheight=11in,paperwidth=9.5in,
			left=1.25in,right=1.25in,
			top=0.75in,bottom=0.75in,
			heightrounded,marginparwidth=1.2in,
			marginparsep=0.05in]{geometry}
\usepackage{xcolor}
\usepackage{xargs} 
\usepackage[textsize=footnotesize]{todonotes}
\newcommandx{\nt}[2][1=]{\todo[linecolor=red,
			backgroundcolor=red!10,bordercolor=red,#1]{ #2}}
\newcommandx{\jy}[2][1=]{\todo[linecolor=green,
			backgroundcolor=green!10,bordercolor=green,#1]{JY: #2}}
\else
\newcommand{\nt}[1]{{}}
\newcommand{\jy}[1]{{}}
\fi

\newif\iftwocolumn
\twocolumntrue

\newtheorem{theorem}{Theorem}[section]
\theoremstyle{definition}

\theoremstyle{remark}

\SetKwProg{Fn}{Function}{}{}
\SetKwComment{Comment}{$\triangleright$\ }{}

\makeatletter
\def\subsubsection{\@startsection{subsubsection}
                                 {3}
                                 {\z@ \hspace*{1mm}}
                                 {0ex plus 0.1ex minus 0.1ex}
                                 {0ex}
                                 {\normalfont\normalsize\itshape}}
\makeatother

\newcommand{\mpp}{\textsc{MRPP}\xspace}

\newcommand{\soc}{\textsc{SOC}\xspace}

\font\titlefont=ptmb at 14.8pt
\title{\titlefont
Targeted Parallelization of Conflict-Based Search for Multi-Robot Path Planning
}
\author{Teng Guo   \qquad Jingjin Yu
\thanks{G. Teng, and J. Yu are with the Department of 
Computer Science, Rutgers, the State University of New Jersey, Piscataway, NJ, USA. 
Emails: {\tt\small \{teng.guo, jingjin.yu\}@rutgers.edu}.
}
}

\begin{document}

\maketitle
\thispagestyle{empty}
\pagestyle{empty}


\begin{abstract}
Multi-Robot Path Planning (\mpp) on graphs,  equivalently known as Multi-Agent PathFinding (MAPF), is a well-established NP-hard problem with critically important applications. 
As serial computation in (near)-optimally solving \mpp approaches the computation efficiency limit, parallelization offers a promising route to push the limit further, especially in handling hard or large \mpp instances. 
In this study, we initiated a \emph{targeted} parallelization effort to boost the performance of conflict-based search for \mpp. Specifically, when instances are relatively small but robots are densely packed with strong interactions, we apply a decentralized parallel algorithm that concurrently explores multiple branches that leads to markedly enhanced solution discovery.
On the other hand, when instances are large with sparse robot-robot interactions, we prioritize node expansion and conflict resolution. 
Our innovative multi-threaded approach to parallelizing bounded-suboptimal conflict search-based algorithms demonstrates significant improvements over baseline serial methods in success rate or runtime. 
Our contribution further pushes the understanding of \mpp and charts a promising path for elevating solution quality and computational efficiency through parallel algorithmic strategies.
%
\end{abstract}

\section{Introduction}\label{sec:intro}
The primary objective of Multi-Robot Path Planning (\mpp) is to determine a set of collision-free paths guiding many robots from a given start configuration to a specified goal configuration. Optimally solving \mpp is widely known as NP-hard~\cite{YuLav13AAAI, Sur10, Yu2015IntractabilityPlanar}.
Given the prevalence of this problem in real-world settings, developing effective algorithms for tackling \mpp is crucial to enable large-scale applications, including warehouse automation~\cite{WurDanMou08, guizzo2008three, Guo2024WellConnectedSA}, formation control~\cite{PodSuk04, SmiEgeHow08}, autonomous vehicles and garages~\cite{Guo2024DecentralizedLP, Guo2023TowardEP}, agriculture~\cite{cheein2013agricultural}, object transportation~\cite{RusDonJen95}, and swarm robotics~\cite{preiss2017crazyswarm}, among others.
Despite being an active research subject since the 1980s~\cite{KorMilSpi84, ErdLoz86, LavHut98b, GuoPar02}, \mpp continues to attract attention as even a few percentages of performance improvement can translate to significant competitive advantages in practice. Recent algorithms, e.g.,~\cite{YuLav16TRO, boyarski2015icbs, cohen2016improved}, balance between computational efficiency and solution optimality.
One notable category of algorithms is bounded-suboptimal search \cite{barer2014suboptimal, li2021eecbs}\jy{Add some references}, which enhances efficiency and scalability with bouned optimality guarantees. Most prior works in bounded suboptimal search are serial; in this work, we leverage parallelization to boost the performance of bounded suboptimal search for multi-robot path planning.

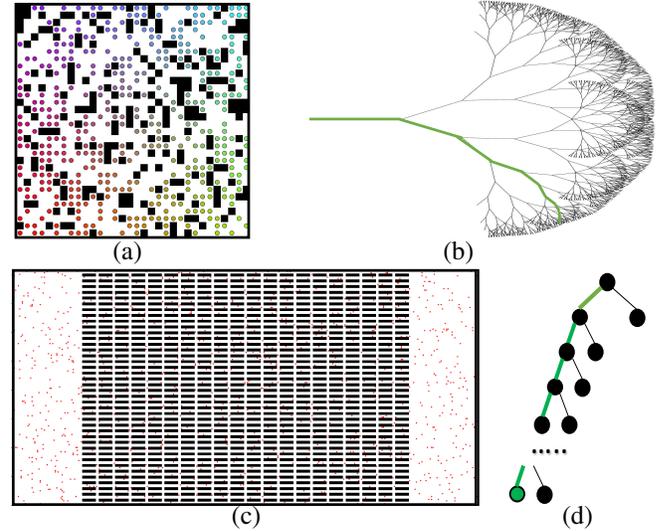
\begin{figure}[h]
    \centering
    \includegraphics[width=1\linewidth]{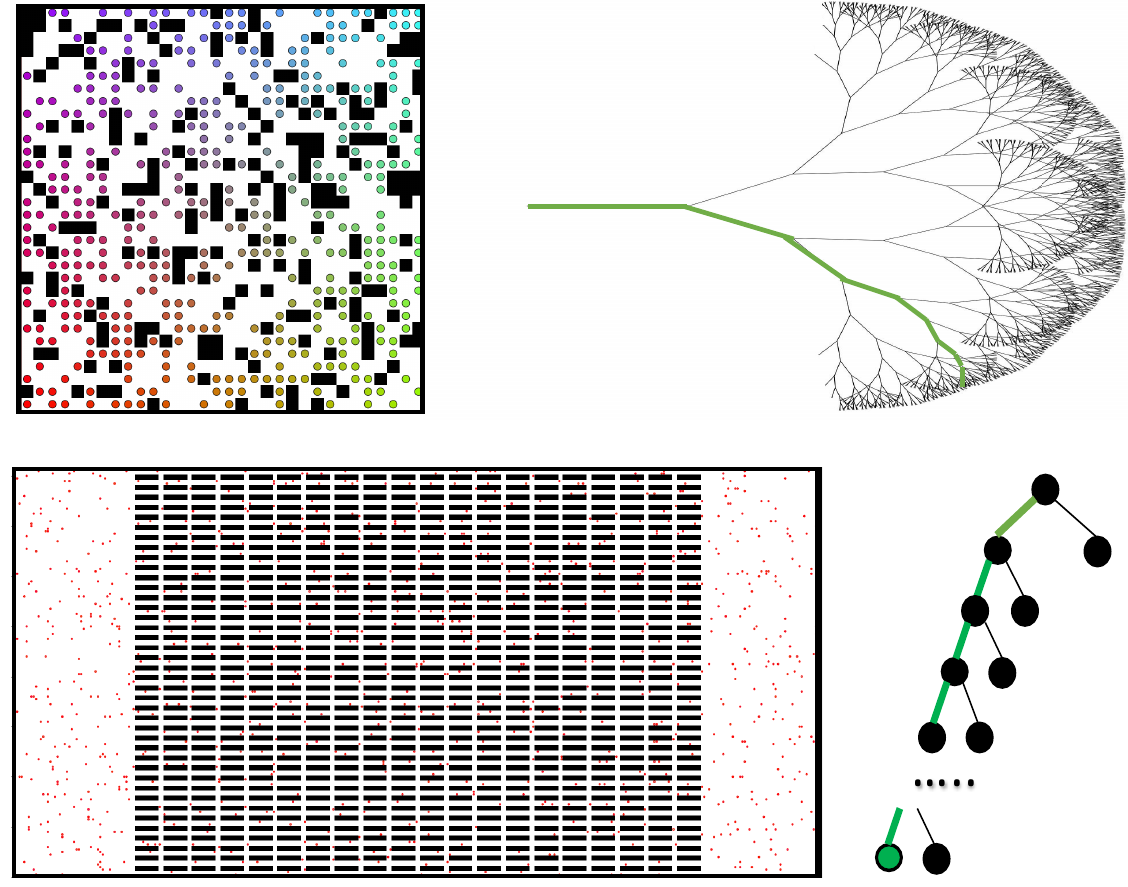}
    \put(-205, 95){(a)}
    \put(-80, 95){(b)}
    \put(-160, -5){(c)}
    \put(-35, -5){(d)}
    \caption{Types of \mpp instances and their illustrative search tree structure. (a) An example with strong robot-robot coupling.
(b) Solving strongly coupled instances demands traversing numerous branches through the search tree, as many branches lead to dead ends.
(c) An example with weak robot-robot coupling.
(d) Solving weakly coupled instances generally does not involve extensive branch exploration, as many branches lead to good quality feasible solutions.}
    \label{fig:introduction}
\end{figure}
Our parallelization effort adopts a targeted and focused approach. We categorize \mpp instances into two types: (1) strongly correlated dense instances and (2) large-scale weakly correlated instances. 
The first category involves relatively small maps with a moderate number of robots but high robot density (very large and dense settings are still currently behind reach), where density can be local or global. In such settings, robot-robot interactions are frequent, characterized by the substantial increase in node expansions required to find a bounded-suboptimal solution compared to the search tree depth. %
The challenge persists even when adopting a large suboptimality bound to encourage robots to address conflicts by taking detours.
Consequently, attempting to resolve conflicts by introducing detours through adding constraints often leads to new conflicts. 
The intricacy renders the entire conflict resolution process time-consuming. 
The second category involves large maps with thousands of robots at moderate robot densities. Robots exhibit weak interactions in these instances, providing numerous opportunities to resolve conflicts by incorporating minor detours. Consequently, the search process does not demand the exploration of numerous branches, and the number of node expansions closely aligns with the tree depth during the search tree expansion. 
\jy{A side note: strictly speaking, robots do not exhibit correlations. The planning processes for two robots are entangled. We should avoid saying ``robots are correlated''.}

We present distinct parallelization strategies tailored to address each category's distinct challenges. In instances characterized by strong robot-robot interactions, where the bottleneck lies in the exponentially growing number of node expansions, we introduce a novel decentralized parallel algorithm that simultaneously explores multiple branches with marked improvement in the likelihood of discovering a solution. Conversely, in the second category, where the number of node expansions is not the limiting factor, but the node expansion process is sluggish, we propose parallel techniques designed to accelerate the node expansion rate and the conflict resolution process.

We demonstrate that the parallel techniques we propose maintain the bounded-suboptimality and completeness guarantees of the original solver. Our experimental findings highlight that the parallelized version exhibits a success rate approximately $40\%$ higher than the serial version in small, strongly-coupled instances while achieving a speedup ranging from $2\times$ to $4\times$ in large weakly-coupled instances.
\jy{Add some highlight of the computational result to affirm the contributions?}

\textbf{Related Work.}
Whereas computing feasible solutions to Multi-Robot Path Planning (\mpp) can be quickly solved~\cite{KorMilSpi84}, optimally solving \mpp induces significant computation complexity~\cite{Sur10,YuLav13AAAI,Yu2015IntractabilityPlanar}, leading to a continuous effort seeking better algorithms. 
Optimal solvers fall into two primary categories: reduction-based solvers, including ILP \cite{YuLav16TRO}, SAT\cite{surynek2012towards}, and ASP\cite{erdem2013general}, and search-based solvers like ICTS\cite{sharon2013increasing}, CBS \cite{sharon2015conflict}, and ODrM*\cite{wagner2015subdimensional}. While optimal solvers offer exact solutions, they often exhibit limited scalability. Numerous heuristic improvements have been proposed to enhance the performance of optimal Conflict-Based Search (CBS), such as those found in \cite{boyarski2015icbs, li2019improved, li2019disjoint}. 
In addition to optimal algorithms, various near-optimal approaches have been introduced. Some of these are bounded-suboptimal variants of CBS, such as ECBS \cite{barer2014suboptimal} and EECBS \cite{li2021eecbs}. Other fast and near-optimal algorithms encompass WHCA* \cite{silver2005cooperative}, PIBT~\cite{okumura2019priority}, LACAM~\cite{okumura2023lacam}, DDM~\cite{han2019ddm}, Push and Swap \cite{luna2011push}, among others. Polynomial-time rule-based algorithms with constant-factor optimality guarantees are proposed to tackle extremely dense settings in \cite{GuoFenYu22IROS, guo2022sub, yu2018constant}.

Previous studies have introduced parallel A* search algorithms, such as PRA* \cite{evett1995massively} and HDA* \cite{kishimoto2009scalable}, which leverage multi-threading, as well as GPU-based parallel A* in \cite{zhou2015massively}.  Limited research has been conducted on parallelizing \mpp solution methods. Deep learning-based methods indirectly execute in parallel on GPUs \cite{li2020graph, sartoretti2019primal} but have yet to demonstrate superior optimality and success rates compared to classical search-based methods. They also are incomplete without falling back on classical methods. 
The approach presented in \cite{guo2021spatial, yu2016optimal} adopts a divide-and-conquer strategy to decompose the original problem into sub-problems concurrently solved on CPUs through multi-threading. Although this method enhances scalability, it comes at the cost of sacrificing the optimality guarantees provided by the original solver.
In the work by Lee et al. \cite{lee2021}, a parallel hierarchical composition conflict-based search is introduced for optimal solving of the \mpp. Their approach addresses subproblems that consider only a subset of robots. These subproblems are concurrently solved and hierarchically combined into larger subproblems until the original \mpp problem is restructured. The method is, however, effective only for a moderate number of robots and is less suitable for larger instances characterized by high coupling.
In this paper, we investigate multi-threading-based methods for parallelizing bounded-suboptimal conflict search-based algorithms and demonstrate that these parallel algorithms significantly outperform baseline sequential algorithms.

\textbf{Organization.}
The rest of the paper is organized as follows. 
Sec.~\ref{sec:problem} covers the preliminaries, including the problem formulation and introduction to the bounded-suboptimal algorithms.
In Sec.~\ref{sec:algorithm}, we introduce the parallelization techniques for conflict-based search algorithms.
In Sec.~\ref{sec:evaluation}, we conduct evaluations of the proposed methods on small-scale strongly correlated instances and large-scale weakly correlated instances.
We conclude in Sec.~\ref{sec:conclusion}.

%


\section{Preliminaries}\label{sec:problem}
\subsection{Multi-Robot Path Planning on Graphs (\mpp)}
Consider a grid graph $G=(V, E)$ with width $w$ and height $h$, the vertex set is $V \subseteq \{(i, j) \mid 1 \leq i \leq w, 1 \leq j \leq h, i, j \in \mathbb{Z}\}$. The graph is 4-way connected, meaning that for a vertex $v = (i, j)$, its neighbors in $G$ are $\mathcal{N}(v) = \{(i + 1, j), (i - 1, j), (i, j + 1), (i, j - 1)\} \cap V$.
The problem involves $n$ robots, denoted as $a_1, \dots, a_n$. Each robot, $a_i$, possesses a unique starting state $s_i \in V$ and a unique goal state $g_i \in V$. The joint start configuration is represented as $\mathcal{S} = \{s_1, \dots, s_n\}$, and the goal configuration is $\mathcal{G} = \{g_1, \dots, g_n\}$.
The primary objective of \mpp is to find a set of feasible paths for all robots. A path for robot $a_i$ is defined as a sequence of $T + 1$ vertices $P_i = (p_i^0, \dots, p_i^T)$, which must satisfy the following conditions:
(i) $p_i^0 = s_i$,
(ii) $p_i^T = g_i$, and
(iii) $\forall 1 \leq t \leq T, p_i^{t - 1} \in \mathcal{N} (p_i^t)$.
In addition to ensuring the feasibility of individual paths, for the paths to be collision-free, the following conditions must be met for all time steps $t$ and all robot pairs $a_i$ and $a_j$:
(i) $p_i^t \neq p_j^t$ (no vertex conflicts),
(ii) $(p_i^{t - 1}, p_i^t) \neq (p_j^t, p_j^{t - 1})$ (no edge conflicts).
Often, two optimization objectives are considered for minimization:
(i) Makespan: $\max_{i} \text{len}(p_i)$;
(ii) Sum-of-costs (\soc): $\sum_{i} \text{len}(p_i)$.
This study focuses exclusively on minimizing the Sum of Costs (\soc) and seeking bounded-suboptimal solutions.

\subsection{Enhanced Conflict-Based Search}
We work on parallelizing Enhanced Conflict-Based Search (ECBS) \cite{barer2014suboptimal} (a variant of Conflict-Based Search \cite{sharon2015conflict}\jy{Add a reference}) given its simplicity, good scalability, and lower memory usage. Further reasons supporting this choice are detailed in Sec.~\ref{sec:evaluation}. ECBS is a two-level, bounded-suboptimal search-based algorithm for solving \mpp. We briefly describe the important and relevant characteristics of ECBS, which systematically explores potential paths while efficiently resolving conflicts that may arise by adding constraints.
\begin{itemize}[leftmargin=*]
    \item \textbf{Initialization}: ECBS begins with an empty \emph{open list} (a priority queue for maintaining the optimality bound) and an empty \emph{focal list} (a priority queue storing nodes from the open list whose optimality lower bound is no more than $w \cdot SOC_{LB}$, where $w$ is the suboptimality bound and $SOC_{LB}$ smallest cost in the open list). The initial high-level node, whose paths are computed without considering robot-robot conflicts and constraints, is initialized as empty and pushed into both lists. The open list prioritizes the high-level node whose paths have a smaller cost, while the focal list prioritizes the high-level node whose paths contain fewer conflicts.

    \item \textbf{High-level Exploration}: ECBS repeatedly pops a node from the focal list and checks for conflict. If no conflict exists, ECBS terminates and returns the solution. Otherwise, ECBS resolves conflicts by generating two child nodes with additional constraints. For example, if there is a vertex conflict $(a_1, a_2, v, t)$, representing robots $a_1$ and $a_2$ meeting at vertex $v$ at timestep $t$, ECBS generates two successors that are initially copies of the current node. New constraints are added for each successor to address the conflict. One with the constraint $\langle a_1, v, t\rangle$, indicating that robot $a_1$ is forbidden to enter vertex $v$ at timestep $t$, and another one with the constraint $\langle a_2, v, t\rangle$. The low-level planners are then triggered to replan the path that satisfies the constraints for each child node.
    
    \item \textbf{Low-level Replanning}: At the low level, ECBS uses state-space focal search \cite{pearl1982studies} to replan paths for conflicting robots while satisfying constraints of the high-level node.
\end{itemize}



\section{Methods}\label{sec:algorithm}
We describe in Sec.~\ref{subsec:31} our ECBS parallelization strategy for strong robot-robot interactions and in Sec.~\ref{subsec:32} our ECBS parallelization strategy when robot-robot interactions are less intense. 

\subsection{Decentralizing High-Level Node  Expansion for Strong Robot-Robot Interactions}\label{subsec:31}
When robot densities are very high, with the corresponding map relatively small, the low-level pathfinding and expansion of high-level nodes require minimum computation cycles, but the strong interactions lead to exponential growth of the number of high-level nodes required to expand to find a feasible solution, as the number of robots increases. 
When dealing with strongly correlated robots, finding highly accurate heuristics to reduce the number of high-level node expansions further becomes challenging, and search-based algorithms are required to explore a tremendous number of branches to obtain the final solution. In such cases, leveraging the computational power of multi-core resources can contribute to significant performance enhancement.

CBS/ECBS are best-first search algorithms. Previous studies, e.g., \cite{fukunaga2018parallel}, demonstrated that decentralized parallel best-first search can outperform centralized parallel search. 
When node expansion is not a bottleneck, frequent locking and unlocking of the shared open list is needed, leading to significant synchronization overhead. As a result, parallelization with a central lock can be slower than its serial version. In such cases, decentralized parallelization strategies have advantages. 
On the other hand, a decentralized best-first search, where each thread maintains an individual open list, may encounter challenges such as duplicate node expansion, optimality insurance, and load-balancing issues. 
Searching over the high-level search tree\jy{What is this?} in high-density settings will never encounter duplicate nodes, and there is no need to re-expand a node. Furthermore, we do not seek the optimal solution; all threads terminate after finding a feasible solution. Due to these reasons, for high-density settings, we propose a decentralized parallel conflict-based search framework that leverages multi-core CPUs to expand multiple nodes simultaneously. 

\begin{algorithm}
\DontPrintSemicolon
\SetKwProg{Fn}{Function}{:}{}
\SetKw{Continue}{continue}
  \Fn{\textsc{ParallelECBS}({$\mathcal{S},\mathcal{G}$})}{
 \caption{Decentralized Parallel ECBS (DP-ECBS)\label{alg:decebs}}
 Initialize $\{\text{OPEN}_i\}$, $\{\text{FOCAL}_i\}$, $\{\text{BUFFER}_i\}$\;
 $n_0\leftarrow \texttt{InitialNode}$()\;
 $\text{incumbent}\leftarrow \text{None}$\;
 push $n_0$ to $\text{OPEN}_0$\;
 parallel execute $\texttt{ThreadECBS}(i)$\;
\Return incumbent;
}
\vspace{2mm}
\Fn{\textsc{ThreadECBS}(i)}{
\While{true}{
\If{incumbent$\neq$ None}{
    break\;
}
$\texttt{Pull}(\text{BUFFER}_i, \text{OPEN}_i)$\;
\If{$ \text{OPEN}_i=\emptyset$}{ continue\;}
$SOC_{LB}\leftarrow \texttt{getSocLB()}$\;
$\text{FOCAL}_i\leftarrow \texttt{getFocal}(i)$\;
$n\leftarrow \text{FOCAL}_i.pop()$\;
\If{$n.conflicts=\emptyset$}{
    incumbent$\leftarrow n$\;
    break\;
}
$C\leftarrow \texttt{Expand}(n)$\;
\ForEach{$c\in C$}{
$j\leftarrow \texttt{DestinateThread}(c)$\;
$\texttt{Send}(\text{BUFFER}_j,c)$\;
}

}
}
\end{algorithm}

Alg.\ref{alg:decebs}, denoted as \emph{Decentralized Parallel ECBS} (DP-ECBS), presents a decentralized parallelized version of ECBS. Each thread $i$ maintains its dedicated local open list (\text{OPEN}$_i$), focal list (\text{FOCAL}$_i$), buffer, and conflict avoidance table (CAT). The algorithm commences by generating the initial node $n_0$ through a sequential single-robot focal search, disregarding inter-robot conflicts. The incumbent solution is set to \texttt{None}, and $n_0$ is pushed onto the open list of the first thread (\text{OPEN}$_0$).
Simultaneously, the \texttt{ThreadECBS} function is invoked in parallel, executing a loop until an incumbent solution is found. If such a solution exists, the loop terminates. Within each iteration, each thread checks its buffer for incoming nodes, pulling and pushing them to its local open list with locking. The local lower bound of the sum of costs is continuously recorded by examining the first node on the open list.
The global lower bound on the sum of costs ($SOC_{LB}$) is calculated as the minimum value among the local lower bounds of all threads, performed atomically to minimize synchronization overheads. The FOCAL list (\text{FOCAL}$_i$) is then updated to include nodes from \text{OPEN}$_i$ with lower bounds not exceeding a threshold ($w \cdot SOC_{LB}$).
A node $n$ is selected from the FOCAL list, removed, and if conflict-free, becomes the incumbent solution, terminating the algorithm. Otherwise, $n$ is expanded to generate a set of child nodes $C$. For each child node $c$ in $C$, the algorithm determines the thread index $j$ using \texttt{DestinateThread} and sends the child node to the corresponding thread's buffer with the thread index being $j$.
While HDA*~\cite{fukunaga2018parallel} employs the Zorbist hash function \cite{zobrist1990new} for load balancing, defining a hash function for high-level nodes in our case is challenging due to the vast search space. Two methods for distributing search nodes to worker threads were tested. In the first, nodes are sent to a \emph{random} thread for statistical load balancing. In the second, for two child nodes generated in thread $j$, one remains in thread $j$, and the other is sent to thread $(j+1) \mod N_p$ where $N_p$ is the number of threads. This \emph{deterministic} approach ensures that each thread receives at least one incoming child in the next iteration.
\jy{This paragraph is too long - please break it into 2-3.}

We only consider solvable \mpp instances (a fast polynomial-time algorithm~\cite{KorMilSpi84} can readily check feasibility). As a result, Alg.~\ref{alg:decebs} does not provide a termination function for the no-solution case.
We now briefly prove that the DP-ECBS remains complete and bounded suboptimal.
\begin{theorem}
  DP-ECBS is complete and $w$-suboptimal.
\end{theorem}
\begin{proof}
  By construction, DP-ECBS does not miss any potential search node of the full search tree. As a direct result, DP-ECBS is complete. To confirm $w$-suboptimality, it suffices to show that the $SOC_{LB}$ obtained by $\texttt{getSocLB}$ is indeed the global lower bound. 
  During the evaluation of $SOC_{LB}$ by a thread $i$, there may be incoming nodes en route generated by other threads. However, the sum of costs in conflict-based search monotonically increases. In other words, the SOC of generated children cannot be less than the SOC of their parent, which is the local lower bound of the corresponding thread. Consequently, the incoming nodes do not have a smaller value than $SOC_{LB}$. It is possible that when a thread $i$ unlocks its buffer and evaluates the global lower bound, another thread $j$ pushes a node to the buffer. The node in the buffer is not considered when $i$ evaluates the global lower bound.  However, thread $j$ has not updated its local lower bound yet when thread $i$ is evaluating the global lower bound, because thread $j$ is behind thread $i$. As a result,  the cost lower bound at the node in the buffer cannot be smaller than $SOC_{LB}$ obtained by $\texttt{getSocLB}.$
\end{proof}

\subsection{Optimizing Bypass and Conflict Counting through Parallel Acceleration for Low to Moderate Robot Density Settings}\label{subsec:32}
When robot densities are low to moderate, but the number of robots is large, search-based algorithms may require more iterations, and each iteration can be potentially slow. At the same time, as illustrated in Fig.~\ref{fig:introduction}, search tree exploration needs not span as many branches to locate a solution as in the strong interaction setting. Consequently, concurrent expansion of multiple nodes is ineffective. While a naive approach can be applied to work with the two children who come from a conflict resolution, such an approach utilizes only two threads and fails to harness the power of multi-core processing fully. We propose an alternative parallelization algorithm, as shown in Alg.~\ref{alg:pbp}, to accelerate the search process by reducing the number of iterations and enhancing the node expansion rate.

\begin{algorithm}
\DontPrintSemicolon
\SetKwProg{Fn}{Function}{:}{}
\SetKw{Continue}{continue}
  \Fn{\textsc{ParallelBP}({$N$})}{
 \caption{Parallel Bypass-augmented ECBS (PB-ECBS)\label{alg:pbp}}
\While{$|N.conflicts|\geq \alpha$ or $\text{still drops}$}{
$\mathcal{C}\leftarrow \texttt{findDisjointConflicts(N)}$\;
execute in parallel \ForEach{$c \in \mathcal{C}$}{ $\texttt{ByPass}(c)$}\;
execute in parallel $N.conflicts\leftarrow \texttt{findConflictsWithHashMap(N)}$\;
$\texttt{updateCAT}(N.paths)$\;
}
}
\end{algorithm}

\textbf{Parallel Bypass:}
The Bypass technique (BP) \cite{boyrasky2015don} is a tool for enhancing performance in conflict-based search. When a node $N$ is validated, and a conflict $\langle a_i, a_j, v, t\rangle$ is identified, BP seeks an alternative path for either conflicting robot $a_i$ or $a_j$. For an alternative path to qualify as a bypass for bounded-suboptimal solvers, its cost must fall within the suboptimality bound without including the conflict. If a conflicting robot $a_i$ in node $N$ has a valid bypass $P_i$, the node is not split into child nodes. Instead, the solution path $P_i$ within $N.solution$ is replaced with the valid bypass path, and the node is re-evaluated in the subsequent iteration without tree expansion.

BP effectively avoids splitting almost half of the nodes, particularly when the suboptimality ratio is not tight. However, in scenarios with many robots, the time required to bypass conflicts can still be substantial. To harness the power of multi-core processing and expedite the process, each iteration of the bypass routine for node $N$ involves selecting a disjoint conflict set $\mathcal{C}$ from $N.conflicts$, where $\forall c_1, c_2\in \mathcal{C}$, the set $\{c_1.a_1, c_1.a_2\}\cap \{c_2.a_1, c_2.a_2\} = \emptyset$. The involvement of different, non-overlapping robots in each conflict within this set allows for parallel processing of disjoint conflicts using BP through multi-threading. During each iteration of parallel BP processing, the conflict-avoidance table for low-level replanning is not updated, and all threads share the same conflict-avoidance table. While this may result in a less accurate conflict heuristic for low-level planning in general, potentially leading to a smaller reduction of the number of conflicts, for large and sparse instances, the replanned path is not significantly altered from the previous one. Thus, the impact would not be significant. Moreover, sharing the same conflict-avoidance table without modification saves time. The parallel BP continues to address disjoint conflicts until the number of conflicts stabilizes or becomes small enough (i.e., less than the number of threads or a predefined constant number $\alpha$). Subsequently, the algorithm returns to the regular search routine to handle any remaining conflicts until the final solution is found. The conflict avoidance table is updated only after addressing these disjoint conflicts.

\textbf{Parallel Conflict-Counting:}
Conflict counting becomes more time-consuming as the number of robots increases. The brute-force method, involving iterating through each pair of robots for every timestep, comes with a computational complexity of $O(n^2T)$, where $T$ represents the makespan. In the context of the EECBS\cite{li2021eecbs}\jy{Add a reference}, a more efficient dynamic programming approach is employed to count conflicts. Since only one robot is replanned during node expansion, the conflicts involving this replanned robot are the only ones that need re-evaluation, resulting in a computational requirement of $O(nT)$. We adopted this strategy in our ECBS search routine.
Following each iteration of parallel BP, where multiple robots are involved, conflicts for almost all robots must be re-evaluated. Instead of relying on the brute-force method, a hashmap is utilized to check if two robots occupy the same state at each step, reducing the computational complexity to $O(nT)$. To harness multi-core processing power, the time span is segmented into $k$ time windows for $k$ threads. Each thread checks conflicts using the hashmap within a specific time window, distributing the computational load efficiently.

We call the algorithm \emph{Parallel Bypass-augmented ECBS} (PB-ECBS). 
PB-ECBS maintain bounded-suboptimality and completeness, the proof of which is straightforward and is omitted for brevity.

\section{Evaluation}\label{sec:evaluation}
We evaluate the performance of DP-ECBS and PB-ECBS with pure C++ implementations.
%
%
All experiments are conducted on an 8-core 16-thread Intel\textsuperscript{\textregistered} Core\textsuperscript{TM} i7-6900K CPU operating at 3.2GHz under Ubuntu 18.04LTS.\footnote{The implementations of parallel ECBS and (serial) ECBS are based on \url{https://github.com/Kei18/mapf-IR/tree/public/mapf}. For EECBS, we use the code from \href{https://github.com/Jiaoyang-Li/EECBS.git}{https://github.com/Jiaoyang-Li/EECBS.git}. Our code will be available on  \url{https://github.com/GreatenAnoymous/ParallelMrpp1}}

\subsection{Performance of DP-ECBS on Strongly Coupled Settings}
In this section, we compare DP-ECBS and PB-ECBS to ECBS~\cite{barer2014suboptimal}, EECBS~\cite{li2021eecbs} and LaCAM~\cite{okumura2023lacam}. All CBS-based algorithms use dynamic programming to store conflicts in the high-level node to avoid redundant computation and employ collision avoidance tables to efficiently evaluate the number of conflict heuristics.

For ECBS and DP-ECBS, vanilla versions are used without any improved reasoning or heuristics. BP is also not used in this section. For EECBS, all available reasoning and heuristics are enabled. The suboptimality bound $w$ is set to 2, as solving a strongly correlated problem using a small suboptimality bound would require much more time. The time limit for each algorithm is set to 120 seconds. For DP-ECBS, 16 threads are utilized, which is the maximum number of concurrent threads for the CPU. Three medium-sized maps with diverse structures are evaluated: orz201d, random-32-32-20, and a $60\times 60$ empty map. For orz201d and random-32-32-20 \cite{stern2019mapf}, starts and goals are uniformly randomly generated. For the $60\times 60$ empty map, robots are constrained to the lower-left corner square area in both start and goal configurations, simulating a constrained multi-robot rearrangement scenario \cite{Guo2023EfficientHF}. We generate 50 instances for each map and evaluate the average computation time, success rate, and SOC suboptimality. If an algorithm cannot solve an instance within 120 seconds, it is considered a failure, and the runtime is counted as 120s. Failed instances will not be considered when evaluating the suboptimality ratio.

The results are presented in Fig.~\ref{fig:orz201d}-\ref{fig:cornerdense}.
LaCAM achieves great success rates and the lowest runtime at the cost of larger suboptimality, especially in very dense scenarios.
DP-ECBS-1 and DP-ECBS-2 represent DP-ECBS using random and deterministic work distribution strategies, respectively. EECBS cannot solve any generated instance in the given range of number of robots on random-32-32-20 within 120 seconds and is not shown in Fig.~\ref{fig:random32}. We observe that both DP-ECBS versions outperform ECBS and EECBS in all scenarios regarding success rate and runtime without resorting to any additional heuristics. On orz201d and in the multi-robot rearrangement scenario, EECBS, with reasoning and improved heuristics enabled, has better success rates than ECBS with fewer robots but is outperformed by ECBS as robot density grows.
\jy{I updated ECBSP1/ECBSP2 to DP-ECBS-1/DP-ECBS-2, please update the figures accordingly. I will look at the figures after your updates.}

The evaluation suggests that when robots are strongly coupled, the reasoning and improved heuristics (in serial methods) that work very well in moderately dense scenarios and conservative suboptimality bounds do not properly capture the interaction among robots. As a result, they do not significantly reduce the number of nodes and branches explored and slow down the node expansion rate. We note that the suboptimality ratio obtained by EECBS is larger than that of DP-ECBS and ECBS. 
EECBS is expected to surpass ECBS when tight suboptimality bounds are enforced; our proposed parallel strategy also applies to EECBS, which we plan to explore in future research. By using multiple concurrent threads to explore different branches and nodes simultaneously, DP-ECBPS methods achieve a much higher success rate than ECBS and EECBS, with a suboptimality ratio close to ECBS. Experiments show random and deterministic work distribution strategies have minor differences.
\begin{figure}[h!]
    \centering
    \includegraphics[width=1\linewidth]{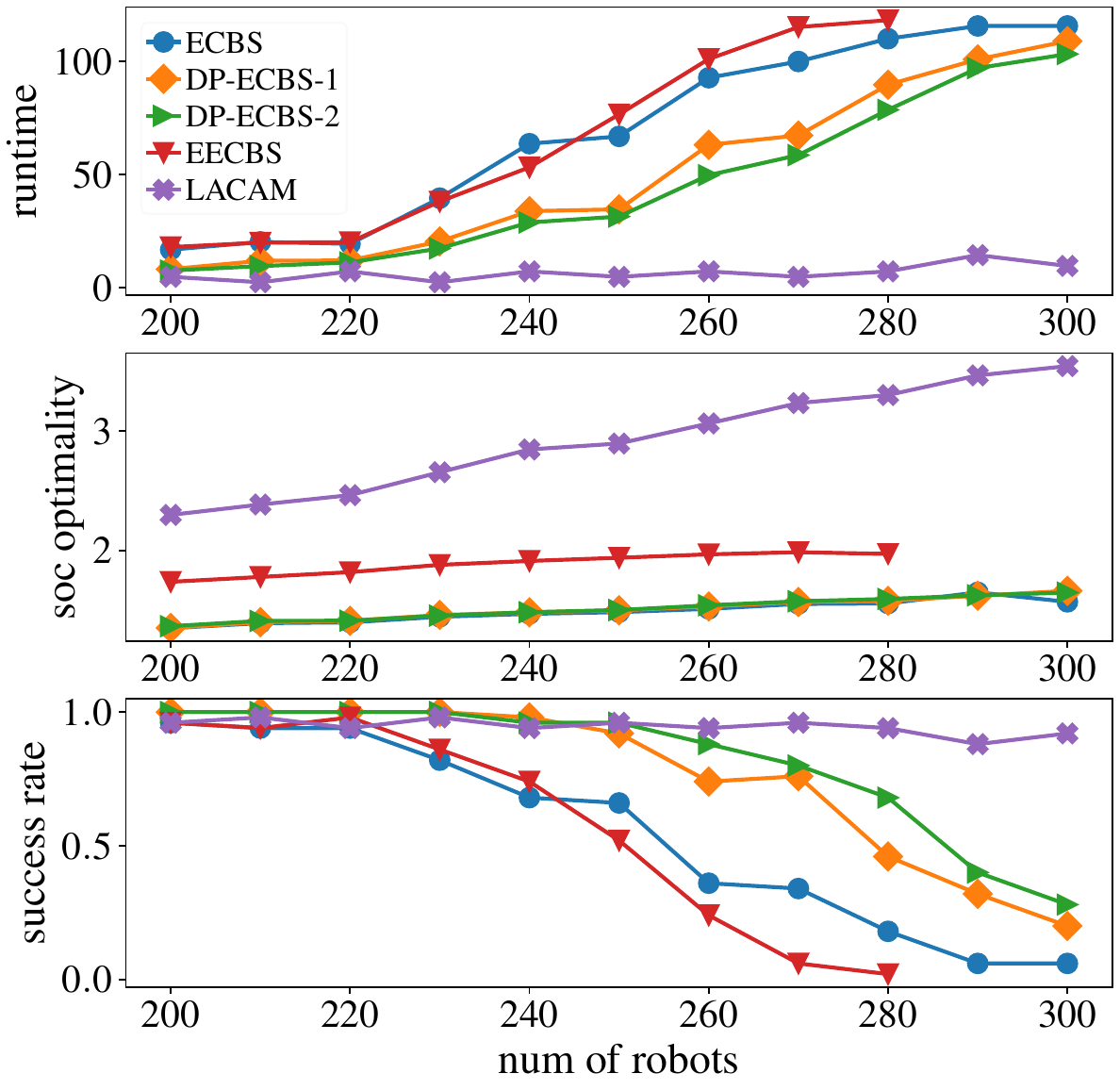}
        \put(-218, 22){\includegraphics[width=0.08\linewidth]{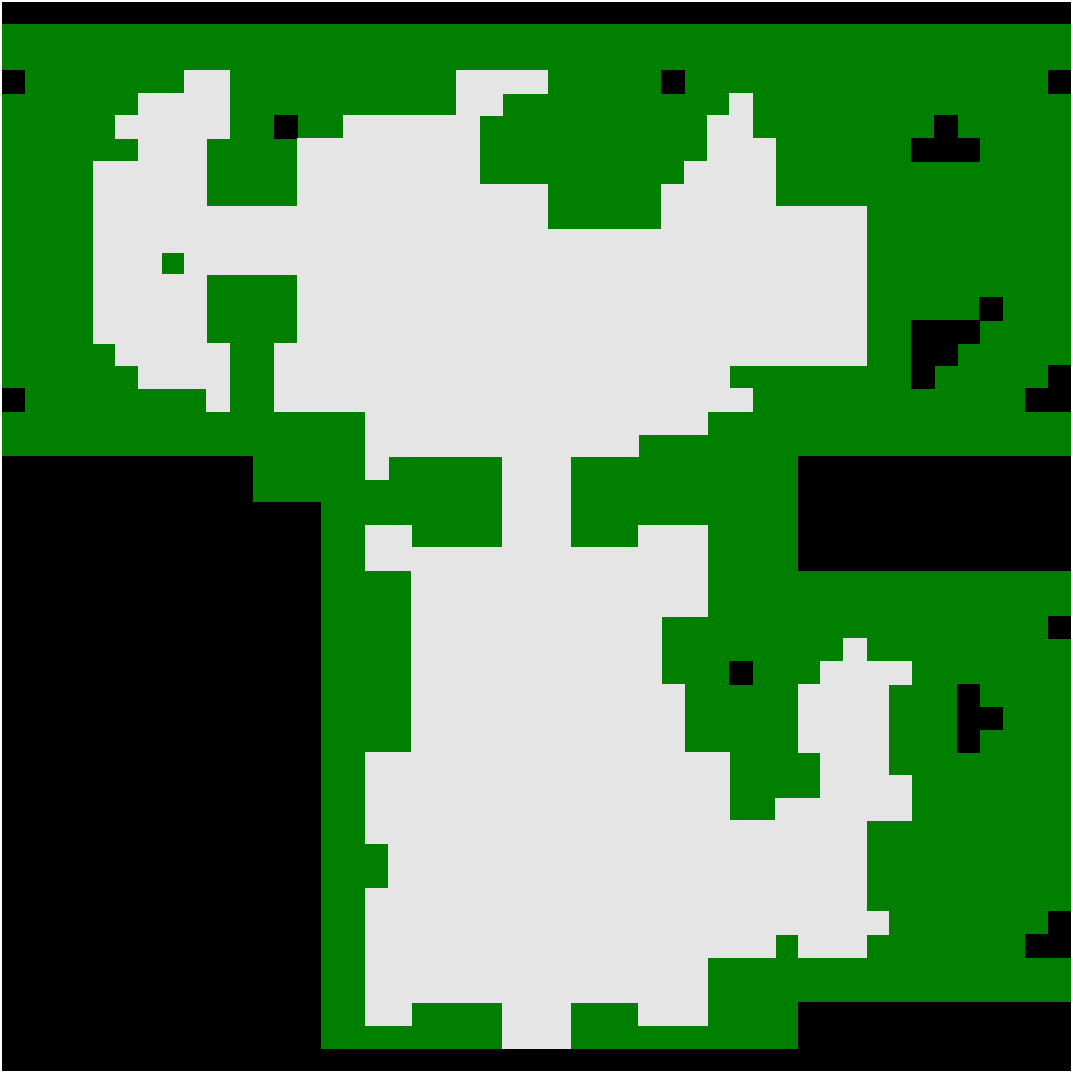}}
    \caption{Experimental results comparing DP-ECBS, ECBS, EECBS, and LaCAM on map orz201d. Metrics include computation time, success rate,  and SOC optimality.}
    \label{fig:orz201d}
\end{figure}

\begin{figure}[h!]
    \centering
    \includegraphics[width=1\linewidth]{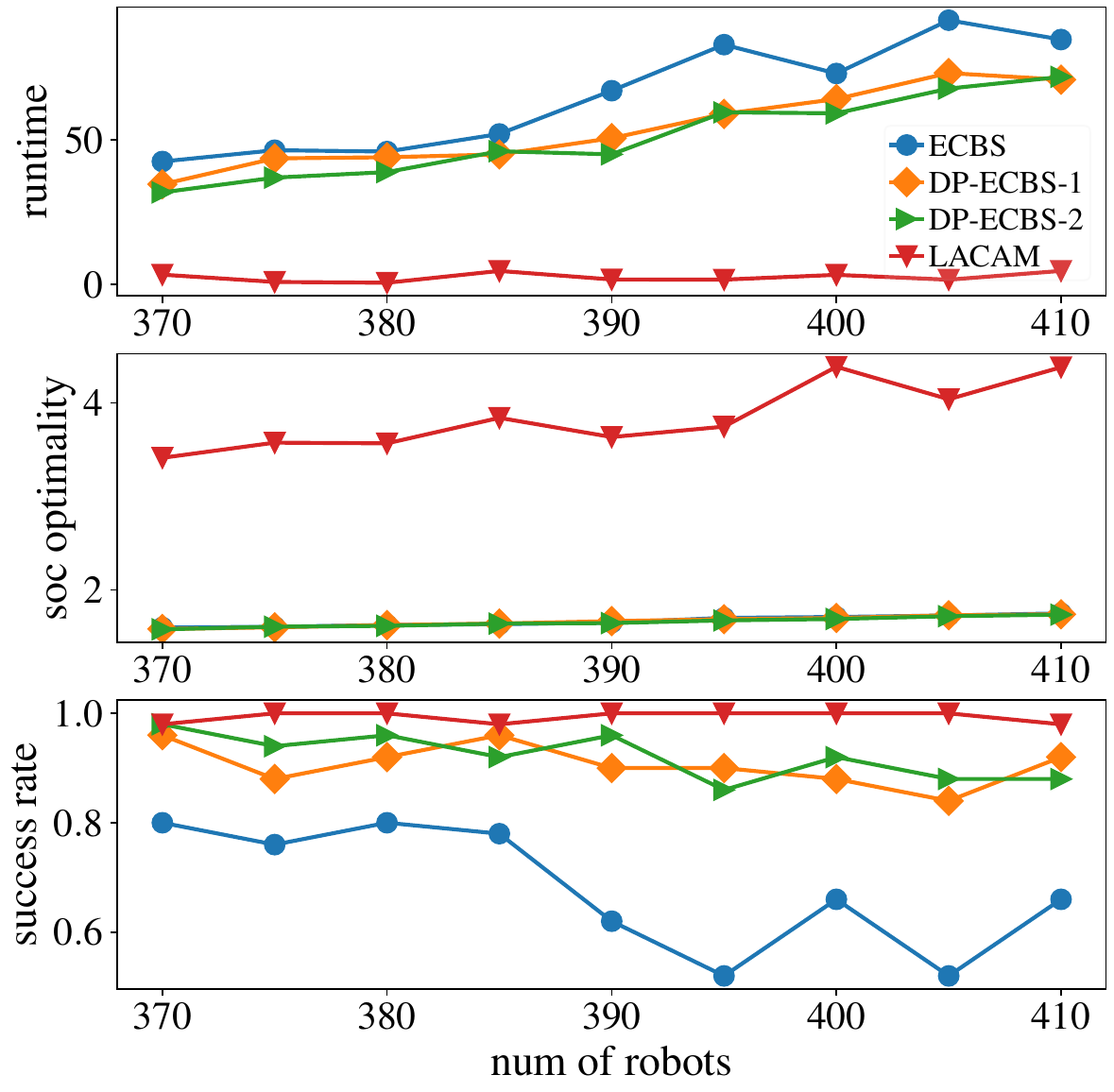}
    \put(-218, 22){\includegraphics[width=0.08\linewidth]{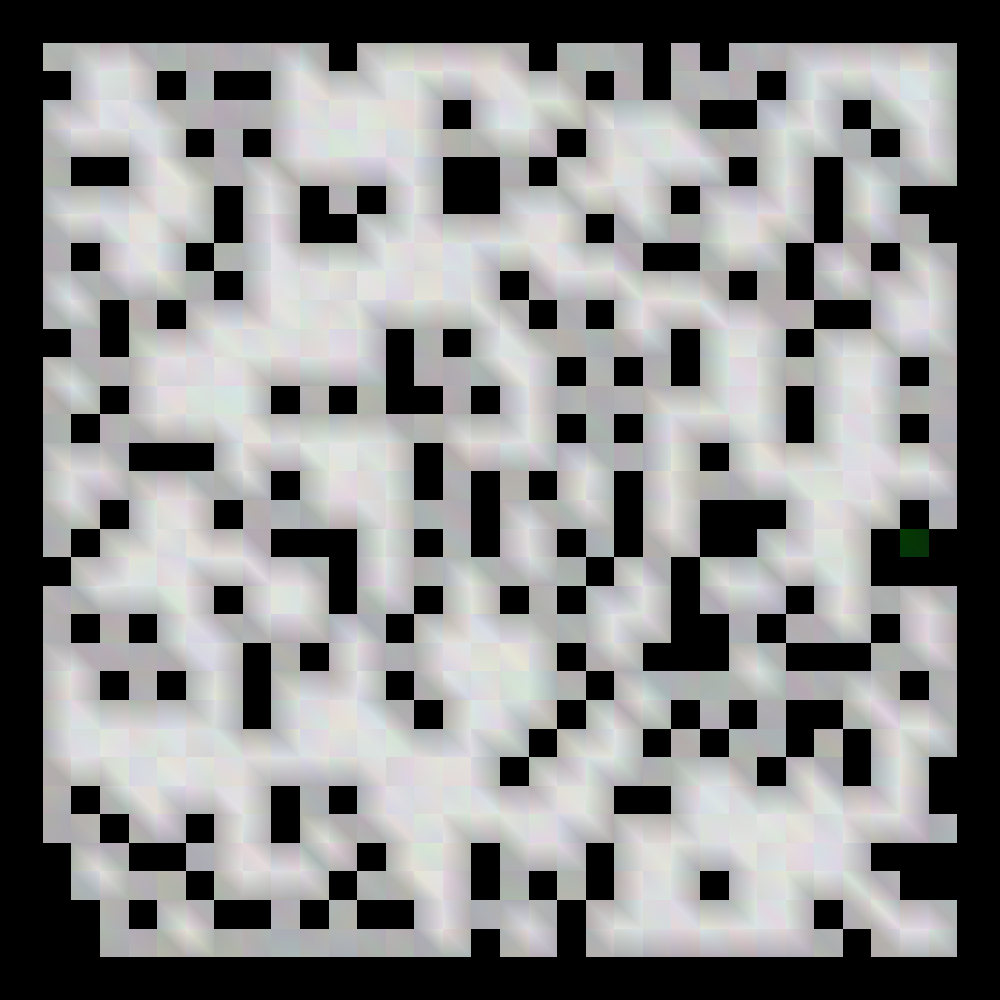}}
    \caption{Experimental results comparing DP-ECBS, ECBS, EECBS, and LaCAM on map random-32-32-20. Metrics include computation time, success rate,  and SOC optimality.}
    \label{fig:random32}
\end{figure}

\begin{figure}[h!]
    \centering
    \includegraphics[width=1\linewidth]{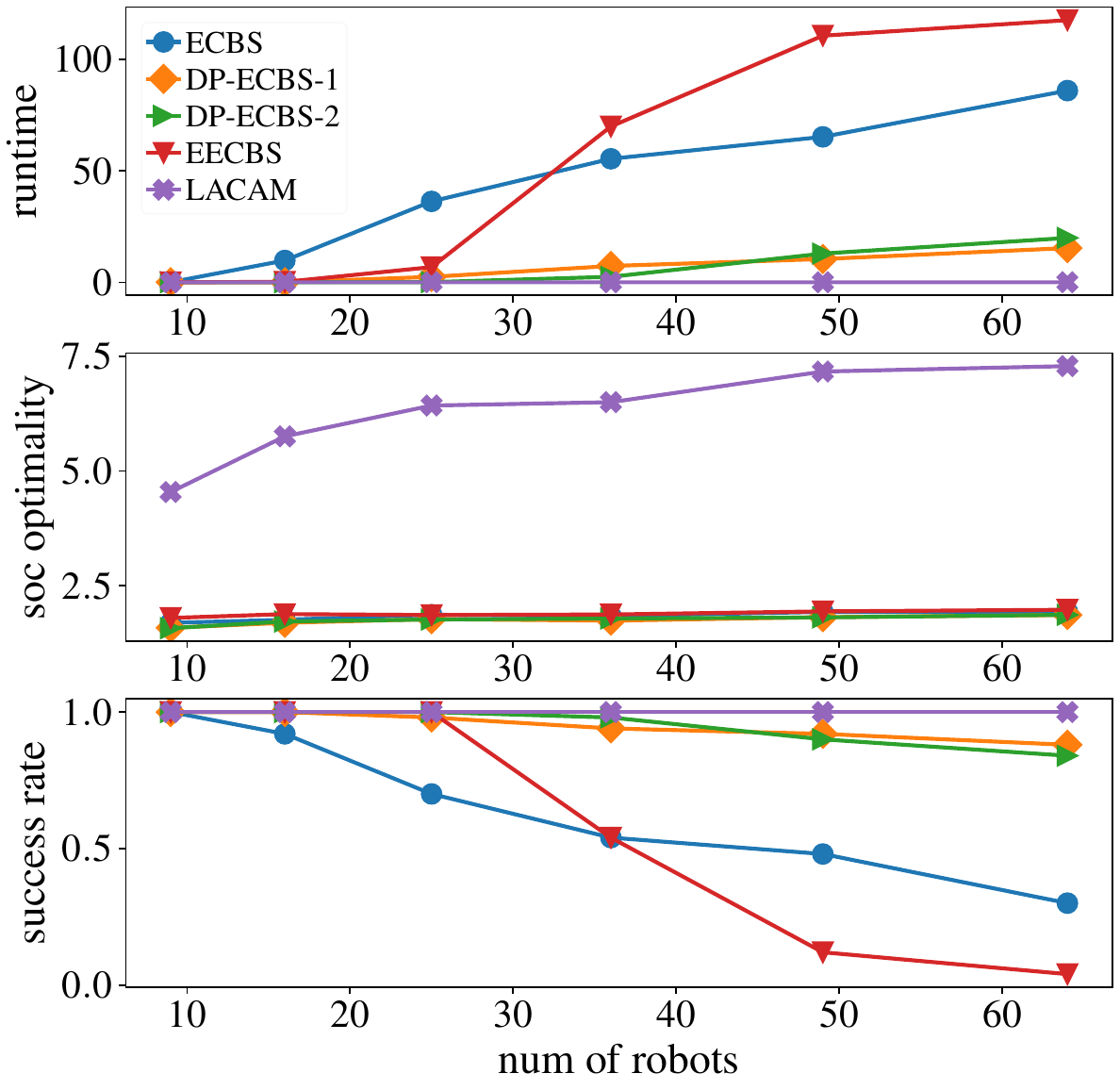}
    \put(-217, 23){\includegraphics[width=0.08\linewidth]{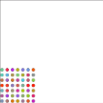}}
    \caption{Experimental results comparing DP-ECBS, ECBS, EECBS, and LaCAM on rearranging robots on a $60 \times 60$ obstacle-free map (an example instance of which is shown in Fig.~\ref{fig:example_rearrangement}. Metrics include computation time, success rate,  and SOC optimality.}
    \label{fig:cornerdense}
\end{figure}

\begin{figure}[h!]
    \centering
    \includegraphics[width=1\linewidth]{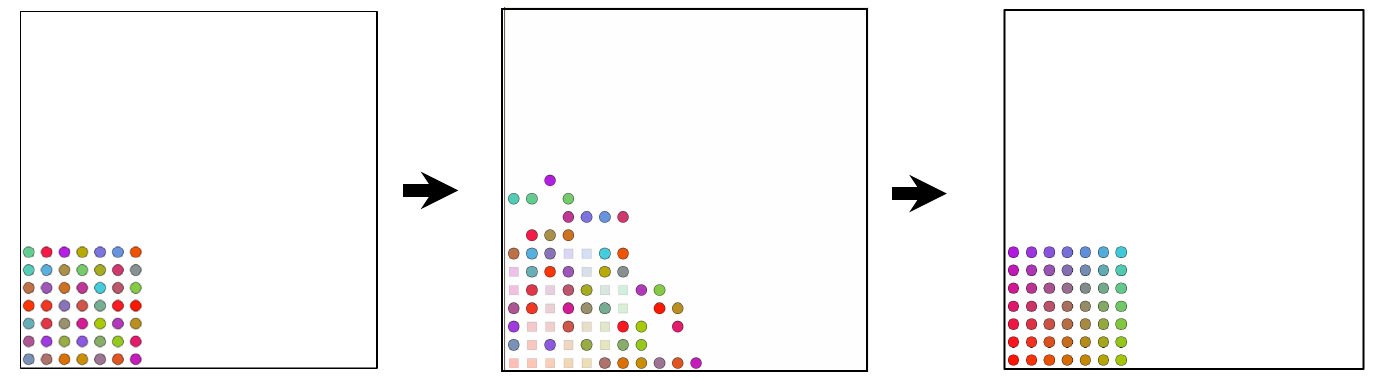}
    \caption{An example instance of dense multi-robot rearrangement.}
    \label{fig:example_rearrangement}
\end{figure}

Next, we explore the correlation between the number of threads employed and the efficiency of DP-ECBS with a random sending strategy. In this analysis, we use a deterministic approach in node transmission. The success rate data across three maps is illustrated in Fig.~\ref{fig:threads_number}. Table~\ref{tab:expansion_rate} shows the average high-level number of expansions for tested algorithms on map orz201d, including the failed instances.

Increasing the number of threads in parallel ECBS corresponds to an improved success rate, providing a greater opportunity to identify promising nodes. However, a trade-off exists as the synchronization overhead becomes more pronounced with a growing number of threads. Notably, a 16-threaded parallel ECBS yields a node expansion rate ten times greater than the ECBS. EECBS, with bypass and other heuristics, expands fewer high-level nodes.

\begin{table}[h!]
\vspace{3mm}
\begin{footnotesize}
  \centering
  \begin{tabular}{|c|c|c|c|c|c|c|}
    \hline
    \textbf{\;\;\# of robots\;\;} &200 &220 & 240 & 260 & 280 &300 \\
    \hline
    \textbf{ECBS} & 795&859& 2168&2444 & 3300 & 2685\\
    \hline
   \textbf{EECBS} & 324& 404&869 &904 &805  & 800\\
    \hline
   \textbf{2 theads}&1023 &1730 & 2358&4597 &4886 &4919  \\
    \hline
      \textbf{4 threads}& 1402 &2023 &5165 &6752&10003 &11562  \\
        \hline
            \textbf{8 threads} &2621 &3293 &7600 &11214 &14803 &17178 \\
            \hline
            \textbf{16 threads} &4627 &6504 &15621 &19320 &22789 &25634 \\
            \hline
  \end{tabular}
  \caption{Average number of expanded nodes on orz201d. }
  \label{tab:expansion_rate}
\end{footnotesize}
\end{table}

\begin{figure}[h!]
    \centering
    \includegraphics[width=1\linewidth]{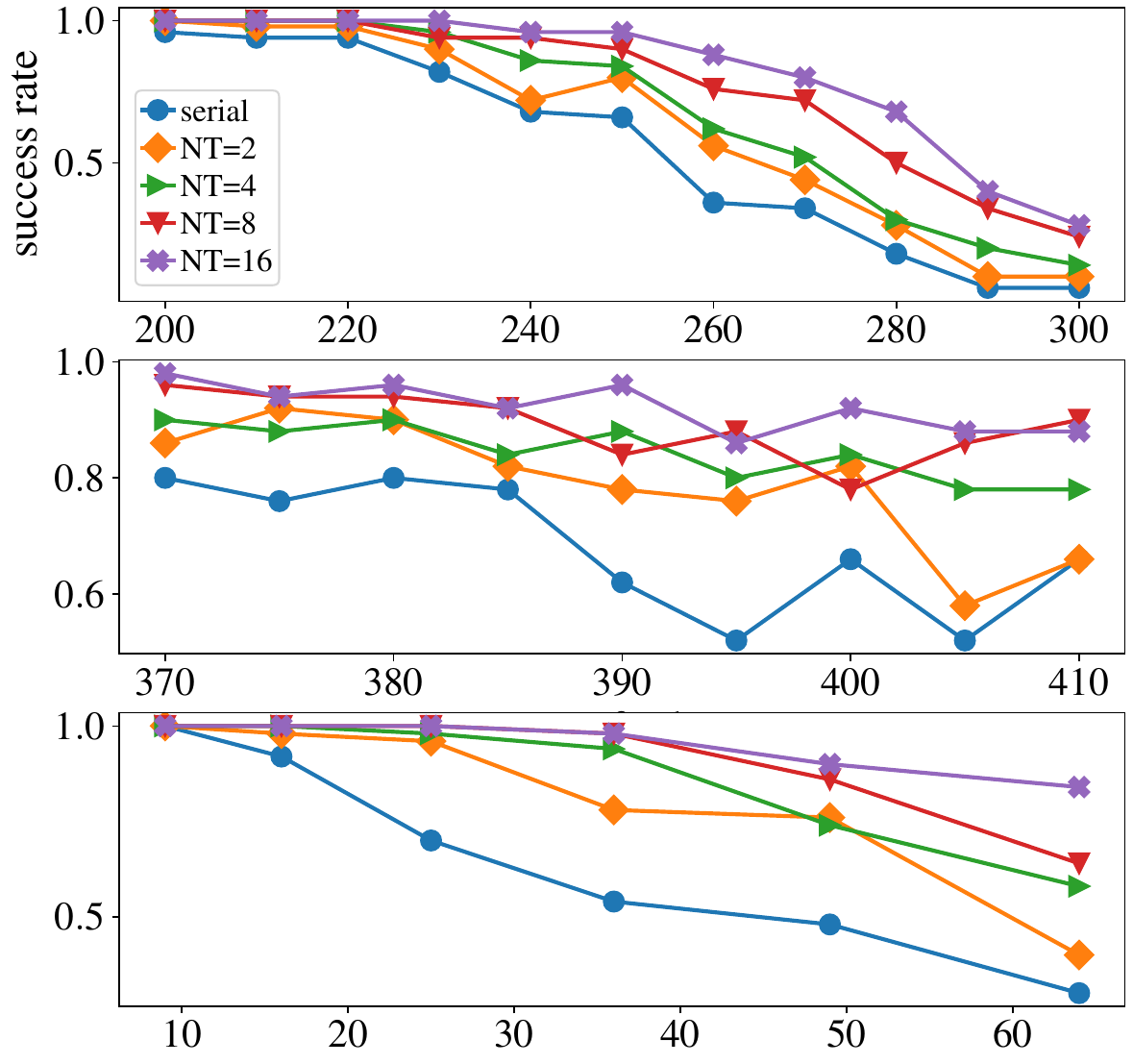}
    \caption{Success rate for DP-ECBS with different number of threads on the three maps. Left: orz201d, Middle: random-32-32-20, Right: empty grids.}
    \label{fig:threads_number}
\end{figure}

\subsection{Performance of PB-ECBS on Large Maps}
We assessed the effectiveness of PB-ECBS using the warehouse-20-40-10-2-2 characterized (size: $340\times164$) and Shanghai-0-256 (size: $256\times 256$) \cite{stern2019mapf}. The evaluation involves testing each map and value of $n$ by generating 50 random instances with a suboptimality bound set at 2.0. The comparison is made among the ECBS algorithm with BP, EECBS, and LaCAM.

The results, as depicted in Fig.~\ref{fig:warehouse_plot}-\ref{fig:shanghai_plot}, highlight the performance of parallel ECBS across different thread counts, juxtaposed with ECBS and EECBS. In this context, $nt=2$ signifies the utilization of 2 threads for parallel acceleration. The findings demonstrate that the parallel algorithm effectively reduces runtime by concurrently addressing and counting conflicts.
\begin{figure}[h!]
    \centering
    \includegraphics[width=1\linewidth]{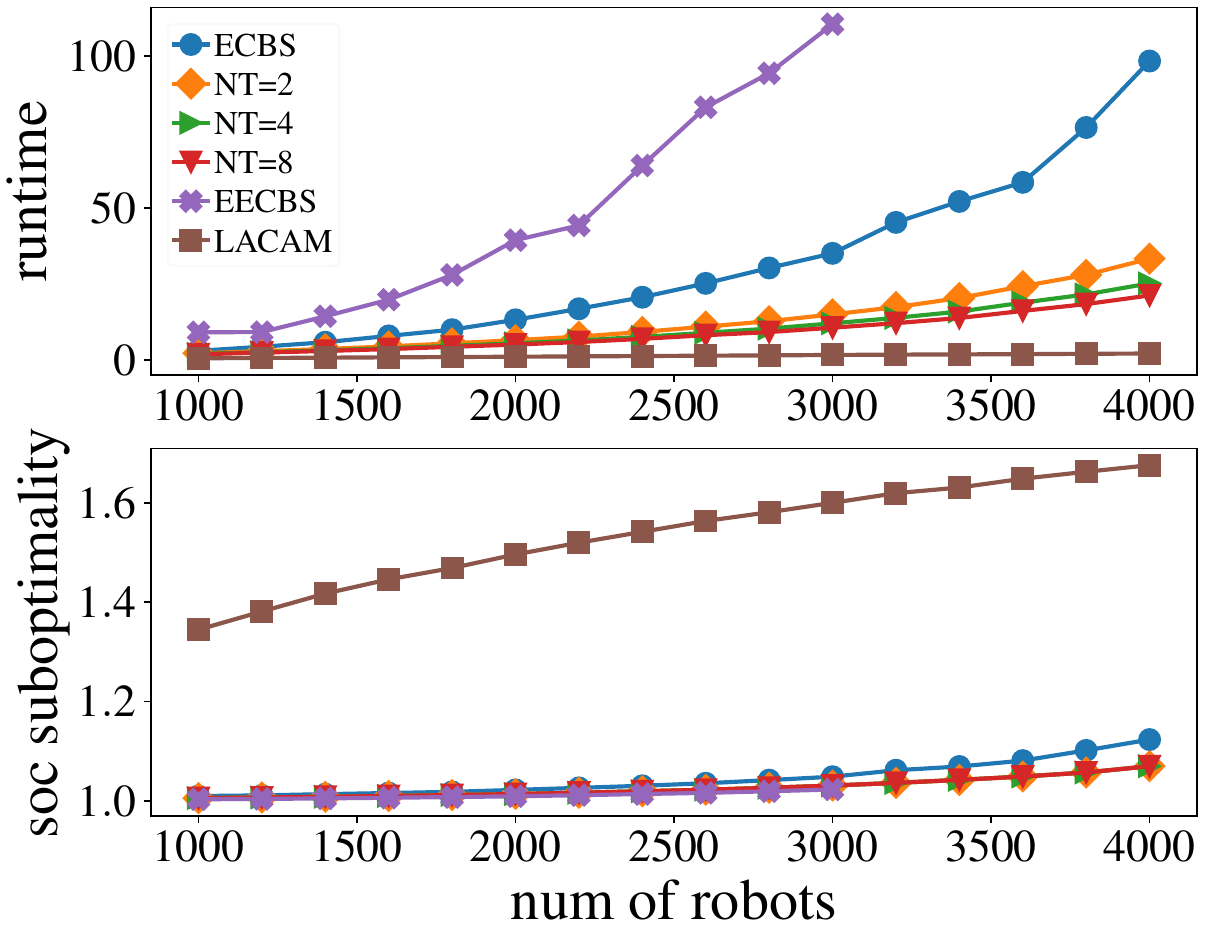}
      \put(-215, 75){\includegraphics[width=0.08\linewidth]{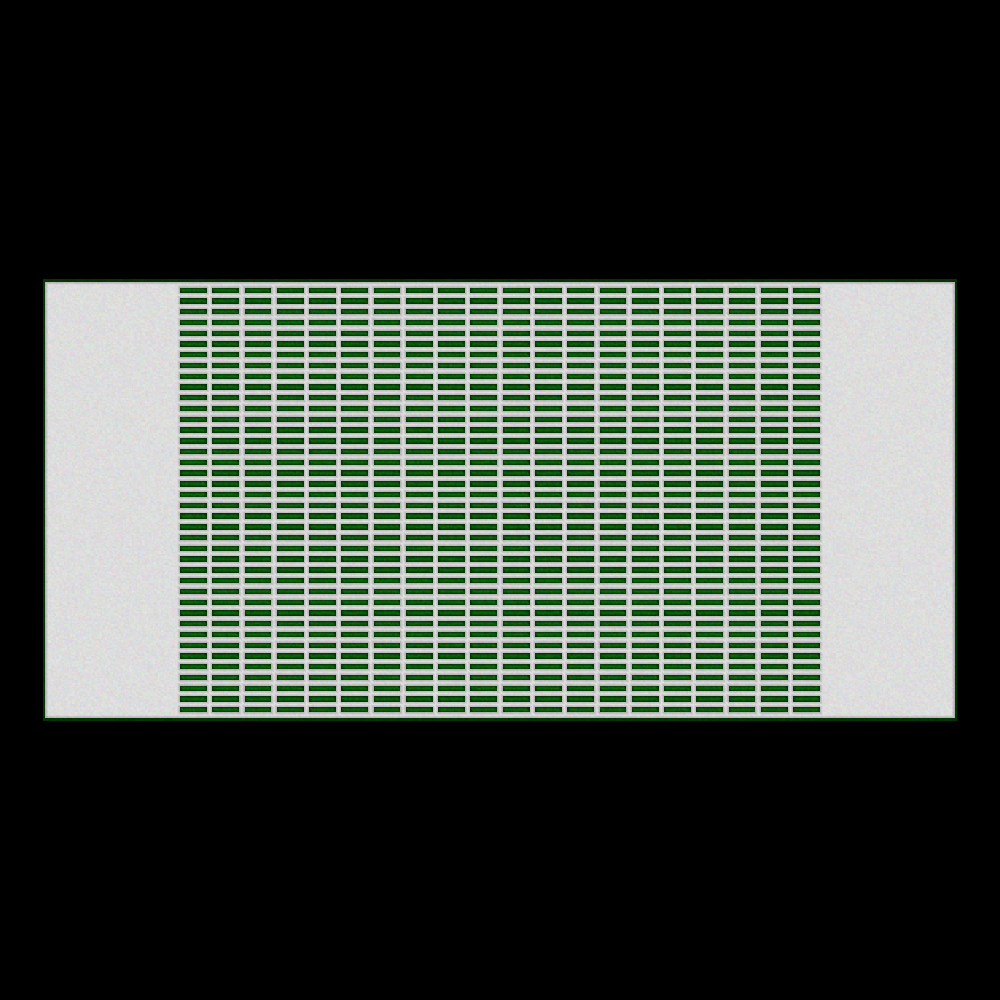}}
    \caption{Experimental results comparing PB-ECBS (1-8 threads), EECBS, and LaCAM on map warehouse-20-40-10-2-2. Metrics include computation time and SOC optimality.}
    \label{fig:warehouse_plot}
\end{figure}
\begin{figure}[h!]
    \centering
    \includegraphics[width=1\linewidth]{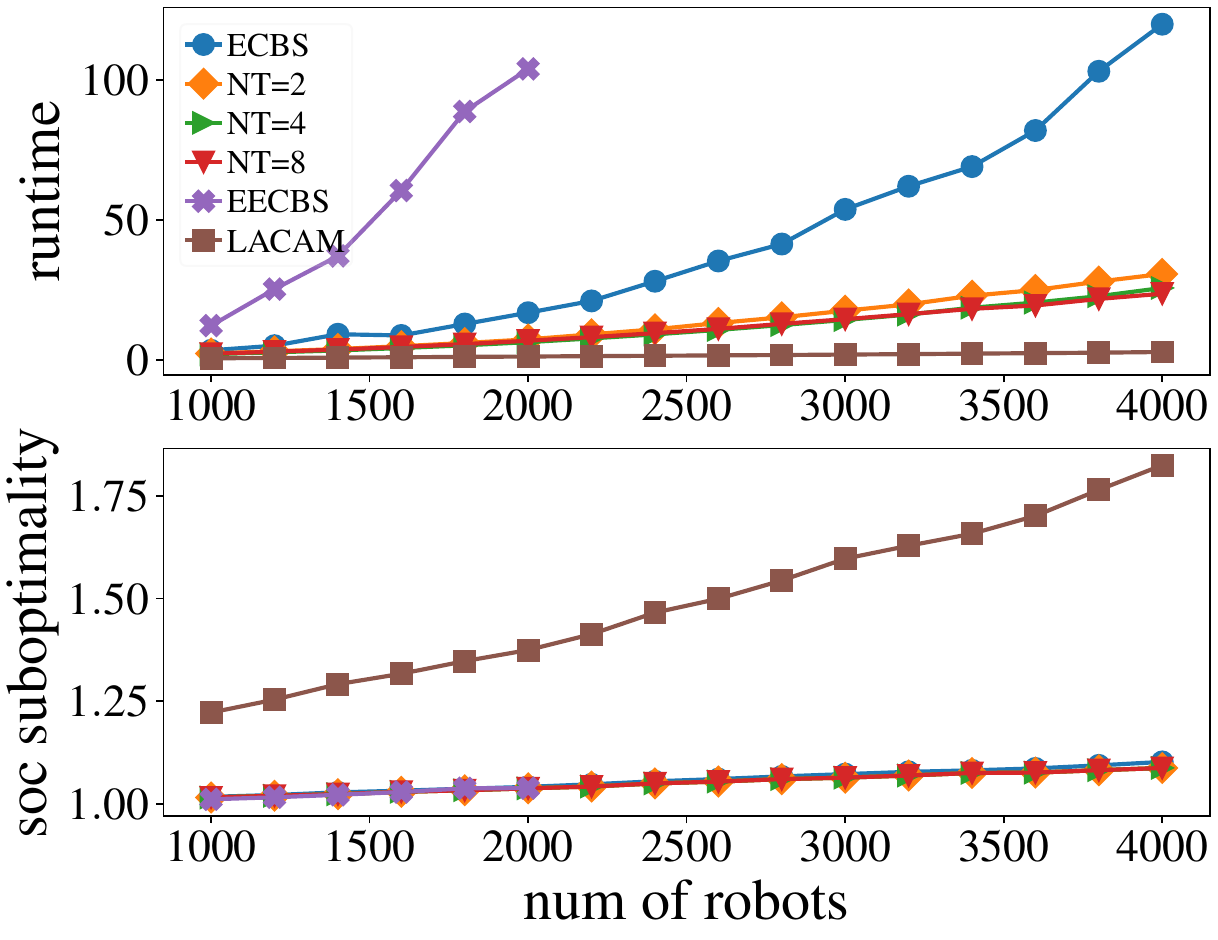}
     \put(-210, 75){\includegraphics[width=0.08\linewidth]{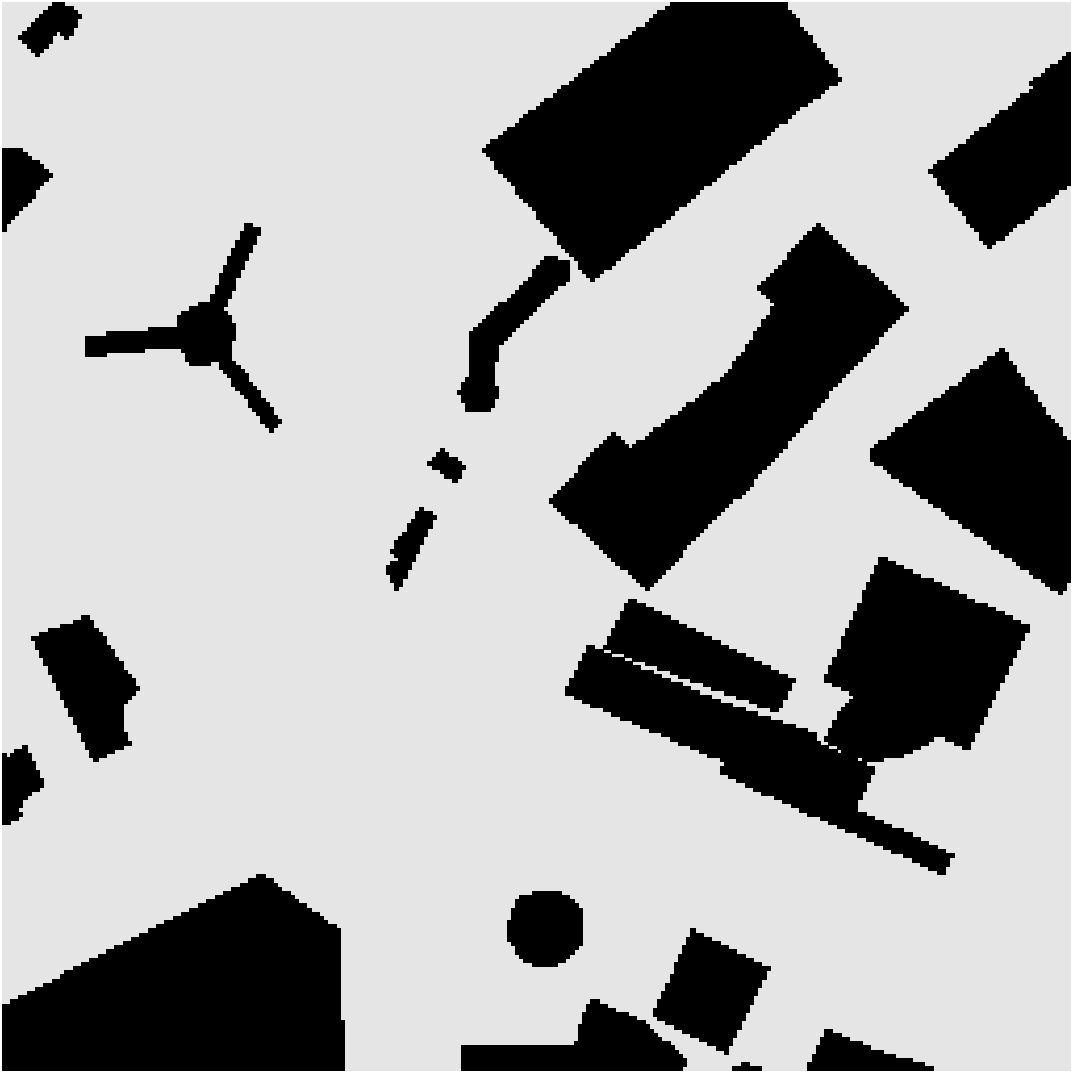}}
    \caption{Experimental results comparing PB-ECBS (1-8 threads), EECBS, and LaCAM on map Shanghai-0-256. Metrics include computation time and SOC optimality.}
    \label{fig:shanghai_plot}
\end{figure}
EECBS exhibits scalability up to approximately 3000 robots, while PB-ECBS can solve problems with 4000 robots in approximately 20 seconds. This results in a notable speedup of $2\times$-$4\times$ compared to ECBS, maintaining a suboptimality ratio of approximately 1.0-1.1. LaCAM achieves fast problem resolution within 2-3 seconds, showcasing excellent scalability. However, its optimality lags significantly behind conflict-best search algorithms.
We note that the actual speedup gain of multi-threading is much less than the ideal value, especially when the number of threads is large, primarily due to a portion of the sequential part that cannot be parallelized.

\section{Conclusion and Discussions}\label{sec:conclusion}
In conclusion, our work introduces simple yet effective parallel techniques aimed at improving the performance of conflict-based search algorithms. We address two main challenges: (1) robots may interact strongly on small- to medium-sized maps and (2) slow node expansion in large-scale instances, and present targeted, innovative parallelization strategies for each challenge separately. We note that our targeted approach can be readily applied in practice because the map size and robot density are generally known \emph{a priori}.

When robots strongly interact (e.g., when robots are highly concentrated), our decentralized parallel algorithm, DP-ECBS, enhances solution discovery by simultaneously exploring multiple branches, mitigating the impact of exponential node expansion. Conversely, for large maps with low to medium robot density, our proposed parallel techniques, PB-ECBS, focus on accelerating node expansion and conflict resolution.
While our contributions bring measurable progress towards more efficient \mpp solutions, it is also crucial to acknowledge the limitations of our methods. Parallel node expansion used by DP-ECBS proves less effective on sparse instances, and PB-ECBS is less efficient for dense instances with conservative suboptimality bounds. Large-scale strongly-coupled instances continue to pose challenges for both parallel methods.
Furthermore, the absence of theories on how to determine the strength of correlation in an instance remains a gap in our understanding.
Future work will involve refining these strategies, exploring adaptive approaches, and extending our methods to dynamic environments for comprehensive applicability.

\bibliographystyle{formatting/IEEEtran}
\bibliography{bib/all}

\end{document}